\documentclass[journal]{IEEEtran}
\IEEEoverridecommandlockouts                 

\usepackage{algorithmic}
\usepackage{array}
\usepackage[caption=false,font=normalsize,labelfont=sf,textfont=sf]{subfig}
\usepackage{textcomp}
\usepackage{stfloats}
\usepackage{url}
\usepackage{verbatim}
\usepackage{graphicx}
\def\BibTeX{{\rm B\kern-.05em{\sc i\kern-.025em b}\kern-.08em
    T\kern-.1667em\lower.7ex\hbox{E}\kern-.125emX}}
\usepackage{balance}
\usepackage{cite}

\usepackage{hyperref}
\usepackage{bm}
\usepackage{amsmath,amssymb,latexsym}
\usepackage{amsthm}
\newtheorem{theorem}{Theorem}

\newtheorem{property}{Property}
\usepackage{booktabs}
\usepackage{longtable}
\allowdisplaybreaks 
\usepackage{tikz}
\usetikzlibrary{graphs}

\usetikzlibrary{positioning}

\title{Structured Deep Neural Network-Based Backstepping Trajectory Tracking Control for Lagrangian Systems}
\author{Jiajun Qian, Liang Xu, Xiaoqiang Ren, Xiaofan Wang
\thanks{The work was supported in part by the National Natural Science Foundation of China under Grant 62373239, 62273223, 62336005 and 62333011, and the Project of Science and Technology Commission of Shanghai Municipality under Grant 22JC1401401. }
\thanks{Jiajun Qian, Xiaofan Wang are with the School of Mechatronic Engineering and Automation, Shanghai University, Shanghai, China. Emails: {\tt\small \{qianjiajun, xfwang\}@shu.edu.cn}}
\thanks{Xiaoqiang Ren is with School of Mechatronic Engineering and Automation, Shanghai University and Key Laboratory of Marine Intelligent Unmanned Swarm Technology and System, Ministry of Education, Shanghai, China. Emails: {\tt\small \href{mailto:xqren@shu.edu.cn}{xqren@shu.edu.cn}}}
\thanks{Liang Xu is with the School of Future Technology, Shanghai University, Shanghai, China.
        Email: {\tt\small\href{mailto:liang-xu@shu.edu.cn}{liang-xu@shu.edu.cn}} (corresponding author)}
}
\begin{document}
\maketitle

\begin{abstract}
  Deep neural networks (DNN) are increasingly being used to learn controllers due to their excellent approximation capabilities.
  However, their black-box nature poses significant challenges to closed-loop stability guarantees and performance analysis.
  In this paper, we introduce a structured DNN-based controller for the trajectory tracking control of Lagrangian systems using backing techniques. 
  By properly designing neural network structures, the proposed controller can ensure closed-loop stability for any compatible neural network parameters.
  In addition, improved control performance can be achieved by further optimizing neural network parameters.
  Besides, we provide explicit upper bounds on tracking errors in terms of controller parameters, which allows us to achieve the desired tracking performance by properly selecting the controller parameters.
  Furthermore, when system models are unknown, we propose an improved Lagrangian neural network (LNN) structure to learn the system dynamics and design the controller.
  We show that in the presence of model approximation errors and external disturbances,  the closed-loop stability and tracking control performance can still be guaranteed.
 The effectiveness of the proposed approach is demonstrated through simulations.
\end{abstract}

\begin{IEEEkeywords}
deep neural networks, backstepping control,  trajectory tracking, stability guarantees, Lagrangian systems.
\end{IEEEkeywords}

\section{Introduction}
Learning-based control methods have gained significant attention in the control community due to the development of machine learning techniques.
Moreover, neural network-based control has now become prevalent due to the excellent function approximation capability of neural networks.
Traditional neural network-based control uses shallow neural networks~\cite{kumpati1990identification}.
In contrast, deep neural networks (DNNs) are superior to shallow NNs in representing function compositions~\cite{rolnick2018power,lin2017does} and avoid the curse of dimensionality~\cite{poggio2017and, b2dc17b4-50bc-32aa-a52a-1db8b75640aa}, therefore motivating the applications of DNNs in control systems~\cite{7139643, lutter2018deep,cranmer2020lagrangian,takeishi2021learning,sanyal2023ramp, chee2022knode,bauersfeld2021neurobem, dawson2023safe,chang2019neural,gaby2022lyapunov,pmlr-v87-richards18a,zhou2022neural}

DNNs in control systems are mainly used to learn dynamics or learn controllers.
In the first category, researchers use DNN to learn complex models~\cite{7139643, lutter2018deep,cranmer2020lagrangian,takeishi2021learning,sanyal2023ramp} or combine them with the first principles to capture the uncertainty and residual terms of the system~\cite{ chee2022knode,bauersfeld2021neurobem}. 
The second application involves using neural networks to learn controllers or control certificates.
However, since DNNs are black-box models, providing formal stability guarantees for neural network-based controllers is difficult.
Many approaches have been proposed to address this challenge.

Neural Lyapunov control methods~\cite{dawson2023safe,chang2019neural,gaby2022lyapunov,pmlr-v87-richards18a,zhou2022neural} use neural networks to learn both a Lyapunov function and a control law.
This framework was first proposed in~\cite{chang2019neural}, which comprises a learner and a falsifier.
The learner generates a Lyapunov candidate, while the falsifier aims to identify points where the Lyapunov candidate fails to satisfy the Lyapunov condition.
Subsequently, the identified points are added to the training dataset and the neural Lypaunov candidate and the neural controller are re-trained.
This process is repeated until the falsifier cannot find any violation points.
The method has been applied in various fields, such as robot control~\cite{dai2021lyapunov} and learning safe control strategies ~\cite{jin2020neural}.
However, identifying points where the Lyapunov candidate does not satisfy the Lyapunov condition is computationally complex.
Moreover, the training and falsification process may undergo several rounds before a Lyapunov function and a neural network controller can be found.
As a result, the neural Lypaunov control method is computationally demanding.

Therefore, effectively reducing the computational complexity while providing formal stability guarantees for DNN based controllers is
challenging.
Structured DNN controllers are proposed for Hamiltonian or Lagrangian systems as an alternative~\cite{xu2022neural, furieri2022distributed, khader2021learning, massaroliOptimalEnergyShaping2022, sanchez-escalonillaStabilizationUnderactuatedSystems2022}.
Closed-loop stability under structured DNN controllers can be guaranteed as long as the neural networks satisfy certain structures.
These methods require only the solution of a simple training problem, significantly reducing computational complexity.
However, there are few structured DNN-based controllers for the trajectory tracking control problem.

Several existing neural network-based controllers are proposed for trajectory tracking control problems~\cite{yang2020neural,9352495,9694590,xu2023design}. The work~\cite{yang2020neural,9352495} use the powerful function approximation capibility of Radial Basis Function Neural Network (RBFNN) to learn controllers. These methods ensure stability by customizing the parameter update rate of RBFNN, but cannot provide performance guarantees. The work~\cite{9694590,xu2023design} use the outputs of the neural networks to determine the control gains of the traditional controller such as kinematic controller and PID controller. By training and optimizing the parameters of the neural network, the neural network can quickly and accurately find the optimal values of the controller parameters. Since the output of neural networks only determine the control gains of the controller, problems with traditional controllers, such as the need to linearize the system model, inability to cope with disturbances, and failure to adapt to model uncertainty or unknown system models, still persist.

In this paper, we present a structured DNN-based controller for the trajectory tracking control of Lagrangian systems.
Our proposed DNN-based controller is constructed using the backstepping technique~\cite{khalil2002nonlinear,hu2012adaptive}.
We demonstrate that closed-loop stability is ensured for any compatible value of the neural network parameters.
Moreover, we explicitly provide an upper bound for tracking errors in terms of the control parameters.
In scenarios where obtaining model information is difficult, we propose a modified Lagrangian Neural Network (LNN)~\cite{cranmer2020lagrangian} to learn the system model.
Moreover, we show that in the presence of model learning uncertainties and disturbances, we can still guarantee a bounded tracking error.
Finally, we substantiate the effectiveness of our proposed tracking controller through a series of simulations.  

This paper is organized as follows.
In Section \ref{sec.Preliminaries}, some preliminaries are provided.
Section \ref{sec.main result} gives the tracking controller design and performance analysis.
Section \ref{sec.LNN based} shows how to use improved LNNs to learn dynamics and control designs that can guarantee closed-loop stability.
Several simulation results are given in Section \ref{sec.experiment}.
Some conclusions are provided in Section \ref{sec.conclusion}.

Notations: $\mathbb{R}$ ($\mathbb{R}^n$) denote the set of real numbers ($n$-dimensional real vectors). $\lambda_{\min}(\cdot)$  denotes the minimum eigenvalue of a symmetric matrix. $I$ denote the identity matrix. $W_{1:k}$ denote the sequence $\{W_1, \ldots, W_k\}$.  $\|\cdot\|$ denotes the standard Euclidean norm.

\section{Preliminaries}\label{sec.Preliminaries}
In this section, we briefly introduce the Euler-Lagrange equation, Fully Connected Neural Network (FCNN), Lagrangian Neural Networks (LNNs)~\cite{cranmer2020lagrangian} and the Input Convex Neural Network (ICNN)~\cite{amos2017input}, which will be used in subsequent sections. 

\subsection{Euler-Lagrange Equation}
The Euler-Lagrange equation describes the motion of a mechanical system, which is
\begin{equation}
\begin{aligned}
\frac{d}{dt}\left( \frac{\partial L}{\partial \bm{\dot{q}}} \right) - \frac{\partial L}{\partial \bm{q}}  =\bm{u}+\bm{\tau^d}, 
 L(\bm{q}, \bm{\dot{q}})=T(\bm{q},\bm{\dot{q}})-V(\bm{q}),
\end{aligned}\label{eq.Lagrangian}   
\end{equation}
where $L$ is the Lagrangian, $T$ is the kinetic energy, $V$ is the potential energy, $\bm{q}$ is the generalized coordinates, $\bm{u}$ is the generalized non-conservative force, and $\bm{\tau^d}$ is the disturbance acting on the system.
 Equation~\eqref{eq.Lagrangian} can be equivalently represented as
\begin{equation}
    \begin{aligned}
        \bm{M}(\bm{q})\bm{\ddot{q}}+\bm{C}(\bm{q},\bm{\dot{q}})\bm{\dot{q}}+\bm{G}(\bm{q})=\bm{u}+\bm{\tau^d},
    \end{aligned}\label{eq.manipulator}
\end{equation}
where $\bm{M}(\bm{q})=\frac{\partial^{2}(T(\bm{q}, \bm{\dot{q}}))}{\partial \bm{\dot{q}}\partial\bm{\dot{q}}}$ is the inertia matrix, $\bm{C}(\bm{q},\bm{\dot{q}})=\bm{\dot{M}}(\bm{q})-\frac{1}{2}\bm{\dot{M}}(\bm{q})^{\top}$ is the Coriolis matrix, and $\bm{G}(\bm{q})=-\frac{\partial (V(\bm{q}))}{\partial \bm{q}}$ is the gravitational force vector.
The Euler-Lagrange equation has the following properties~\cite{lewisRobotManipulatorControl2004}.
\begin{property}    \label{property. M definit}
$\bm{M}$ is positive definite and is bounded by
    \begin{align*}
        a_1\|\bm{x}\|^2\leq\bm{x}^{\top}\bm{M}\bm{x}\leq a_2\|\bm{x}\|^2, \quad \forall \bm{x} \in \mathbb{R}^n
    \end{align*}
where $a_1,a_2 \in \mathbb{R}^{+}$ are positive constants.
\end{property}
\begin{property}    \label{property. M-2C}
$\bm{\dot{M}}-2\bm{C}$ is skew symmetric satisfying $\bm{x}^{\top}(\bm{\dot{M}}-2\bm{C})\bm{x}=0, \quad \forall \bm{x} \in \mathbb{R}^n$.
\end{property}

\subsection{Fully Connected Neural Network}
Fully Connected Neural Network (FCNN), also known as Multilayer Perceptron (MLP), is a DNN structure.
It contains an input layer, one or more hidden layers, and an output layer.
Each neuron in each layer is connected to every neuron in the adjacent layer.
FCNN has the following expression
\begin{equation}
    \begin{aligned}
      \bm{y_{i+1}}&= \sigma_i\left(\bm{W_i}\bm{y_i}+\bm{b_i}\right), \quad  i=0,\ldots,k-1,\\
      f(\bm{x};\bm{\gamma})&=\bm{y_k},
\end{aligned} \label{eq.standard FCNN}
\end{equation}
where $\bm{y}_i$ are the layer output with $\bm{y}_0=\bm{x}$ being the neural network input, $\sigma_i$ represents the layer activation function, $\bm{\gamma}= \{ \bm{W_{0:k-1}},\bm{b_{0:k-1}}\}$ are the trainable parameters of FCNN. 

\subsection{Lagrangian Neural Network}
Lagrangian Neural Network (LNN)~\cite{cranmer2020lagrangian} is used to learn the Lagrangian function of mechanical systems from data.
Unlike conventional supervised learning, LNN is constructed to respect the Euler-Lagrange equation in an unsupervised manner.
By endowing the laws of physics, LNN has better inductive biases.

In LNN, the Lagrangian function is approximated by an FCNN $\bm{\mathcal{L}}(\bm{q},\bm{\dot{q}};\bm{\gamma})$ with the parameters $\bm{\gamma}$.
During training,  $N+1$ samples  $\{\bm{q}(i),\bm{\dot{q}}(i), \bm{\ddot{q}}(i),\bm{u}(i)\}_{i=0}^{N}$  are collected.
Then the LNN is trained by solving the following optimization problem
\begin{equation}
    \begin{aligned}
    \min_{\bm{\gamma}} \quad & \sum_{i=0}^{N}\left(\bm{\hat{\ddot{q}}}(i)-\bm{\ddot{q}}(i)\right)^{\top}\bm{Q}\left(\bm{\hat{\ddot{q}}}(i)-\bm{\ddot{q}}(i)\right)\\
   \mathrm{s.t.},\quad &\bm{\hat{\ddot{q}}}(i)=\frac{\partial^{2}\mathcal{L}(\bm{q}(i),\bm{\dot{q}}(i);\bm{\gamma})}{\partial \bm{\dot{q}}(i)\partial\bm{\dot{q}}(i)}^{-1} [\bm{u}(i)\\
   &+\frac{\partial \mathcal{L}(\bm{q}(i),\bm{\dot{q}}(i);\bm{\gamma})}{\partial \bm{q}(i)} -\frac{\partial^{2} \mathcal{L}(\bm{q}(i),\bm{\dot{q}}(i);\bm{\gamma})}{\partial \bm{q}(i)\partial\bm{\dot{q}}(i)}\bm{\dot{q}}(i)]
    \end{aligned}\label{eq.LNNs optimization}
\end{equation}
where $\bm{Q}=\bm{Q}^{\top}>0$ is the weight matrix, and the Jacobian and Hessian matrix of $\mathcal{L}$ are calculated using automatic differentiation~\cite{baydinAutomaticDifferentiationMachine2017}.
By solving ~\eqref{eq.LNNs optimization}, one can obtain the optimal parameter $\bm{\gamma^*}$.
Then, using automation differentiation, we can obtain an approximation of $\bm{M}$ as $\bm{\hat{M}}=\frac{\partial^{2}\mathcal{L}(\bm{q},\bm{\dot{q}};\bm{\gamma^*})}{\partial \bm{\dot{q}}\partial\bm{\dot{q}}}$, $\bm{C}$ as $\bm{\hat{C}}=\bm{\dot{\hat{M}}}(\bm{q})-\frac{1}{2}\bm{\dot{\hat{M}}}(\bm{q})^{\top}$, and $\bm{G}$ as $\bm{\hat{G}}=-\frac{\partial \mathcal{L}(\bm{q},\bm{\dot{q}};\bm{\gamma^*})}{\partial \bm{q}}+\frac{1}{2}\bm{\dot{\hat{M}}}(\bm{q})^{\top}\bm{\dot{q}}$.

To effectively solve the optimization problem~\eqref{eq.LNNs optimization}, a special initialization strategy is necessary~\cite{amos2017input}, which is defined according to the depth and width of the LNN
\begin{equation}
    \begin{aligned}
    \left.\nu=\frac{1}{\sqrt{n}}\left\{\begin{array}{cc}2.2&\text{First layer}\\0.58i&\text{Hidden layer } i\in\{1,\ldots\}\\n&\text{Output layer}\end{array}\right.\right.,
\end{aligned} \label{eq. LNNs initial}
\end{equation}
where $i$ represents the $i$-th layer of the neural network, $n$ is the number of neurons in this layer.  The parameters of every layer are then initialized according to $\mathcal{N}(0,\nu^2)$, where $\mathcal{N}$ represents a Gaussian distribution. 

\subsection{Input Convex Neural Network}
Input Convex Neural Network (ICNN) is a feedforward neural network architecture with constraints on the neural network parameters to ensure that the output of the neural network is convex with respect to all inputs (Fully Input Convex Neural Networks or FICNN) or with respect to a subset of inputs (Partially Input Convex Neural Network or PICNN). 

FICNN uses the following architecture for $ i=0,\ldots,k-1$
\begin{equation}
    \begin{aligned}
      \bm{y_{i+1}}= \sigma_i\left(\bm{W_i^{(y)}}\bm{y_i}+\bm{W_i^{(x)}}\bm{x}+\bm{b_i}\right) ,
      f(\bm{x};\bm{\theta})=\bm{y_k},
\end{aligned} \label{eq.standard FICNN}
\end{equation}
where $\sigma_i$ represents layer activation functions, $\bm{x}$ is the FICNN input, $\bm{y_0}= 0,\bm{W_0^{(y)}}= 0$, $\bm{\theta}= \{ \bm{W_{1:k-1}^{(y)}}, \bm{W_{0:k-1}^{(x)}},\bm{b_{0:k-1}}\}$ are network parameters.
One can show that $f(\bm{x};\bm{\theta})$ is convex in $\bm{x}$  if $\bm{W_{1:k-1}^{(y)}}$ are nonnegative and the activation functions $\sigma_i$ are convex and non-decreasing~\cite{amos2017input}. 

PICNN uses the following architecture for $ i=0,\ldots,k-1$
\begin{equation}
    \begin{aligned}
      &\bm{v_{i+1}}=\widetilde{\sigma}_i(\bm{\widetilde{W}_i}\bm{v_{i}+\bm{\widetilde{b}_i}}),\\
      &\bm{y_{i+1}}= \sigma_i\left(\bm{W_i^{(y)}}\left(\bm{y_i}\circ(\bm{W^{(yv)}_i}\bm{v_{i}}+\bm{b^{(y)}_i})\right)\right .\\
     &\quad \left .+\bm{W_i^{(x)}}\left(\bm{x}\circ(\bm{W^{(xv)}_i}\bm{v_{i}}+\bm{b^{(x)}_i})\right)+\bm{W_i^{(v)}}\bm{v_i}+\bm{b_i}\right) ,\\
      &f(\bm{\widetilde{x},\bm{x}};\bm{\theta})=\bm{y_k},
\end{aligned} \label{eq.standard PICNN}
\end{equation}
where $\sigma_i$ and $\widetilde{\sigma}_i$ represent layer activation functions, $\bm{y_0}= 0$, $\bm{v_0}= \bm{\widetilde{x}}$, $\bm{W_0^{(y)}}= 0,$ $\bm{\theta}= \{ \bm{W_{1:k-1}^{(y)}}, \bm{W_{1:k-1}^{(yv)}},\bm{W_{0:k-1}^{(x)}},\bm{W_{0:k-1}^{(xv)}} $$,\bm{\widetilde{W}_{0:k-1}},\bm{b^{(y)}_{1:k-1}},$  $\bm{b_{0:k-1}}$ $,\bm{b_{0:k-1}}^{(x)},\bm{\widetilde{b}_{0:k-1}}\}$ are the network parameters and $\circ$ denotes the Hadmard product.
One can show that $f(\bm{\widetilde{x}},\bm{x};\bm{\theta})$ is convex in $\bm{x}$  if $\bm{W_{1:k-1}^{(y)}}$ are nonnegative and the activation functions are convex and non-decreasing~\cite{amos2017input}.
It should be noted that ICNN can only guarantee the convexity and not the strong convexity.

\section{Main Results}\label{sec.main result} 
In this section, we propose a neural backstepping controller (NBS controller) for the trajectory tracking control of~\eqref{eq.manipulator} under the assumption that model information $\bm{M},\bm{C},\bm{G}$ is available.
The design of the tracking controller is given in Section~\ref{sec.controlDesign}.
Its performance is analyzed in Section~\ref{sec.performanceAnalysis} and the final learning optimization problem is formulated in Section~\ref{sec.learningOptFormulation}.

\subsection{Neural Backstepping Tracking Control Design}\label{sec.controlDesign}
Suppose that the reference trajectory is $\bm{q}^d(t)$, which is continuously differentiable.
We aim to design a controller such that the state $\bm{q}(t)$ follows $\bm{q^d}(t)$.
We define the following errors
\begin{equation}
    \begin{aligned}
\bm{z_1}=\bm{q}-\bm{q^{d}}, \bm{z_2}=\bm{\dot{q}}-\bm{\phi},
\end{aligned}\label{eq.errors}
\end{equation}
where $\bm{\phi}$ is a virtual signal to be designed.
The controller $\bm{u}$ and the virtual signal $\bm{\phi}$ are defined as \begin{equation}
       \begin{aligned}
  \bm{u}=&\bm{G}(\bm{q})+\bm{M}(\bm{q})\bm{\dot{\phi}}+\bm{C}(\bm{q},\bm{\dot{q}})\bm{\phi} \\ &\qquad-\frac{\partial \bm{\Phi}(\bm{z_1};\bm{\theta_1},\bm{S})}{\partial \bm{z_1}}-\bm{D}(\bm{z_2};\bm{\theta_2},m)\bm{z_2},\\
  \bm{\phi}=&\bm{\dot{q}^{d}}- \frac{\partial \bm{\Phi}(\bm{z_1};\bm{\theta_1},\bm{S})}{\partial \bm{z_1}}.  
       \end{aligned} \label{eq.trcaking controller}
     \end{equation}

     In the above control design,  $\bm{\Phi}$ is constructed as
\begin{equation}
    \begin{aligned}
        \bm{\Phi}(\bm{z_1};\bm{\theta_1},\bm{S}) = \bm{\psi}(\bm{z_1};\bm{\theta_1})+\bm{z_1}^{\top} \bm{S} \bm{z_1},
    \end{aligned}\label{eq.Phi}
\end{equation}
where $\bm{S}$ is a positive definite matrix; $\bm{\psi}(\bm{z_1};\bm{\theta_1})$ is a FICNN~\eqref{eq.standard FICNN} with input $\bm{z_1}$ and parameters $\bm{\theta}_1$.
We let the parameters $\bm{W_{1:k-1}^{(y)}}$ in $\bm{\psi}(\bm{z_1};\bm{\theta_1})$  be the output of the ReLu function, whose input is some free parameters.
Moreover, we set $\bm{b_i}$ in $\bm{\psi}(\bm{z_1};\bm{\theta_1})$ to $0$, which ensures $\bm{\psi}(0)=0$. 
Furthermore, by selecting sReLu~\cite{kolter2019learning} as the activation function of $\bm{\psi}$, we can ensure that $\bm{\psi}(z_1) \geq 0$, which further guarantees that $\bm{\Phi}$ has only one minimum at $\bm{z_1}=0$, satisfying $\bm{\Phi}(0)=0$.
This design ensures that $\bm{\Phi}$ is strongly convex with respect to the input $\bm{z_1}$.

In~\eqref{eq.trcaking controller}, $\bm{D}(\bm{z_2}; \bm{\theta_2},m)$ is a positive definite matrix, constructed from deep neural networks with parameters $\bm{\theta_2}$, hyperparameter $m$, and input $\bm{z_2}$.
We first use two independent FCNNs to generate the diagonal and off-diagonal elements of a lower triangular matrix $\bm{T}$.
The activation function for each FCNN is chosen as $\tanh(\cdot)$.
We then filter the output of the diagonal elements of $\bm{T}$ through a ReLu function to ensure that the diagonal elements of $\bm{T}$ are nonnegative.
In addition, a positive number $m$ is added to the diagonal elements of $\bm{T}$.
Then $\bm{D}(\bm{z_2}; \bm{\theta_2},m)$ is constructed as $\bm{D}=\bm{T}^\top \bm{T}$.
Through this design, $\bm{D}$ is guaranteed to be positive definite.

   The structure of the tracking controller is illustrated in Fig. \ref{fig:controller diagram}, the Jacobian $\bm{\dot{\phi}},\frac{\partial\bm{\Phi}(\bm{z_1})}{\partial\bm{z_1}}$ are obtained through automatic differentiation.

\begin{figure}
    \centering
 \includegraphics[scale=0.8]{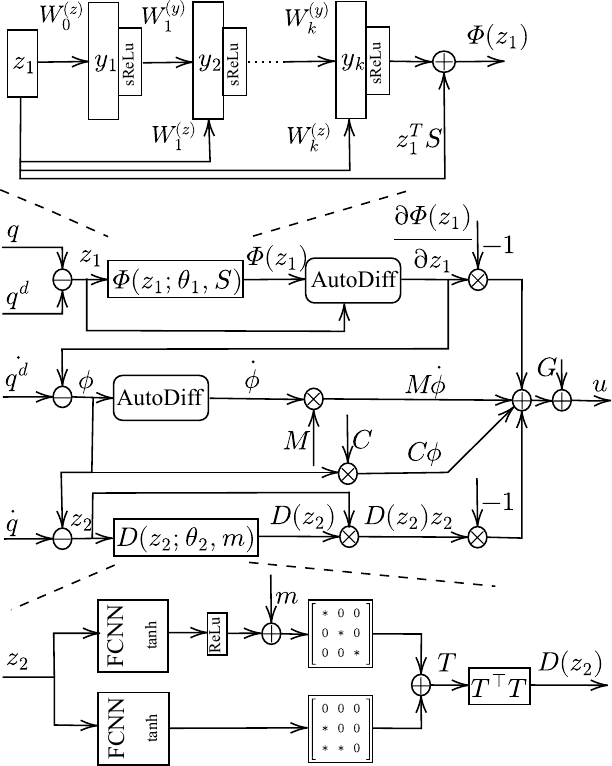}
    \caption {The structure for proposed NBS controller.}
    \label{fig:controller diagram}

\end{figure}
  
\subsection{Stability of the NBS Controller}\label{sec.performanceAnalysis}
In the following theorems, we analyze the performance of the NBS tracking controller. First, we consider the case without disturbances.

\begin{theorem}
    Consider the system~\eqref{eq.manipulator} without disturbance $\bm{\tau^d}$, if the controller is designed as \eqref{eq.trcaking controller}, where $\bm{\Phi}(\bm{z_1})$ is strongly convex in $\bm{z_1}$ with a unique minimum at $\bm{z_1}=\bm{0}$ satisfying $\bm{\Phi}(0)=0$, and $\bm{D}(\bm{z_2})$ is positive definite, then the closed-loop system is globally asymptotically stable at $\bm{z_1}=\bm{0},\bm{z_2}=\bm{0}$.
    \label{Th. stability}
\end{theorem}
\begin{proof}For simplifing notations, we shall ignore the DNN parameters $\bm{\theta_1},\bm{\theta_2}$ and hyperparameters $\bm{S},m$ in \eqref{eq.trcaking controller}.    
    According to \eqref{eq.errors} and \eqref{eq.trcaking controller}, ~\eqref{eq.manipulator} can be reformulated as 
\begin{equation}
    \begin{aligned}
\bm{M}\bm{\dot{z}_2}+\bm{C}\bm{z_2}+\bm{G}=\bm{M}\bm{\ddot{q}}-\bm{M}\bm{\dot{\phi}}+\bm{C}\bm{\dot{q}}-\bm{C}\bm{\phi}+\bm{G}\\=\bm{u}-\bm{M}\bm{\dot{\phi}}-\bm{C}\bm{\phi}=\bm{G}-\frac{\partial \bm{\Phi}}{\partial \bm{z_1}}-\bm{D}(\bm{z_2})\bm{z_2},
    \end{aligned}\label{eq.model error}
\end{equation}

Consider the following candidate Lyapunov function
\begin{gather*}
  \bm{V}(\bm{z_1},\bm{z_2})=\bm{\Phi}(\bm{z_1})+\frac12 \bm{z_2}^{\top}\bm{M}\bm{z_2}.
\end{gather*}
Since $\bm{\Phi}(\bm{z_1})$ is strongly convex in $\bm{z_1}$ and has only one minimum at $\bm{z_1}=0$ with $\bm{\Phi}(0)=\bm{0}$, in view of Property \ref{property. M definit}, we have $\bm{V}(0,0)=0,\bm{V}(\bm{z_1},\bm{z_2})>0, \forall \bm{z_1}\neq0,\bm{z_2}\neq0$, and $\|\bm{z_1}\|,\|\bm{z_2}\| \to \infty \Rightarrow \bm{V}(\bm{z_1},\bm{z_2}) \to \infty$. 
The time derivative of $\bm{V}$ is
\begin{align*}
  &\bm{\dot{V}}=\bm{z_2}^{\top}\bm{M}\bm{\dot{z}_2}+\frac12 \bm{z_2}^{\top}\bm{\dot{M}}\bm{z_2}+\bm{\dot{z}}_1^{\top}\frac{\partial \bm{\Phi}}{\partial \bm{z_1}}\\
         &=\bm{z_2}^{\top}\left( \bm{u}-\bm{M}\bm{\dot{\phi}}-\bm{C}\bm{\phi}-\bm{C}\bm{z_2} -\bm{G} +\frac12 \bm{\dot{M}}\bm{z_2}\right)+\bm{\dot{z}_1}^{\top}\frac{\partial \bm{\Phi}}{\partial \bm{z_1}}\\
         &=\bm{z_2}^{\top}\left( -\frac{\partial \bm{\Phi}}{\partial \bm{z_1}}-\bm{D}\bm{z_2}  \right)+\bm{\dot{z}_1}^{\top} \frac{\partial \bm{\Phi}}{\partial \bm{z_1}}+\frac{1}{2}\bm{z_2}^{\top}(\bm{\dot{M}}-2\bm{C})\bm{z_2}.
\end{align*}
Based on the definitions of~\eqref{eq.errors}, we can derive $\bm{z_2}=\bm{\dot{q}}-\bm{\dot{q}^{d}}+ \frac{\partial \bm{\Phi}}{\partial \bm{z_1}}=\bm{\dot{z}_1}+\frac{\partial \bm{\Phi}}{\partial \bm{z_1}}$. 
According to Property \ref{property. M-2C} and the positive definiteness of $\bm{D}$, we have
\begin{gather*}
  \bm{\dot{V}}=-\frac{\partial \bm{\Phi}}{\partial \bm{z_1}}^{\top} \frac{\partial \bm{\Phi}}{\partial \bm{z_1}}-\bm{z_2}^{\top}\bm{D}\bm{z_2}<0.
\end{gather*}
Therefore, we can conclude that $\bm{z}_1, \bm{z}_2$ will converge to $\{\bm{z_1},\bm{z_2}|\frac{\partial \bm{\Phi}}{\partial \bm{z_1}}^{\top} \frac{\partial \bm{\Phi}}{\partial \bm{z_1}}+\bm{z_2}^{\top}\bm{D}\bm{z_2}=0\}$.
Since $\bm{\Phi}(\bm{z_1})$ is strongly convex in $\bm{z_1}$ and is with a minimum at $\bm{z_1}=\bm{0}$, $\bm{D}$ is positive, we have that $\frac{\partial \bm{\Phi}}{\partial \bm{z_1}}^{\top} \frac{\partial \bm{\Phi}}{\partial \bm{z_1}}+\bm{z_2}^{\top}\bm{D}\bm{z_2}=0$ if and only if $\bm{z_1}=0$ and $\bm{z_2}=0$. Therefore, we have
$  \lim_{t\rightarrow \infty}\bm{z_1}=\bm{0},\lim_{t\rightarrow \infty}\bm{z_2}=\bm{0}.$
\end{proof}

Theorem \ref{Th. stability} demonstrates the stability of the closed-loop system under the NBS tracking controller when there are no disturbances. It is unconditionally stable for all compatible neural network parameters.
That is, the controller can guarantee closed-loop stability as long as the neural networks $\bm{\Phi}$ and $\bm{D}$ satisfy the conditions in Theorem~\ref{Th. stability}, irrespective of the neural network parameters.
In the following theorem, we will show that the NBS tracking controller can achieve a bounded tracking error in the presence of bounded disturbances. 

\begin{theorem}
Consider the system~\eqref{eq.manipulator} with a bounded disturbance $\|\bm{\tau^{d}}\|^2\le d$. If the controller is designed as~\eqref{eq.trcaking controller} and $\bm{D}(\bm{z_2})\ge \frac12 \bm{I}$, then the tracking error will converge to $\{\bm{z_1}|\frac{\partial \bm{\Phi}}{\partial \bm{z_1}}^{\top} \frac{\partial \bm{\Phi}}{\partial \bm{z_1}}\le \frac12 d \}$.
\label{Th. disturbance}
\end{theorem}
\begin{proof} Consider the candidate Lyapunov function
    \begin{align*}
        V(\bm{z_1},\bm{z_2})=\bm{\Phi}(\bm{z_1})+\frac{1}{2}\bm{z_2}^{\top}\bm{M}\bm{z_2}.
    \end{align*}
    The time derivative of $V$ is 
\begin{align*}
    \bm{\dot{V}}&=-\frac{\partial \bm{\Phi}}{\partial \bm{z_1}}^{\top} \frac{\partial \bm{\Phi}}{\partial \bm{z_1}}-\bm{z_2}^{\top}\bm{D}(\bm{z_2})\bm{z_2} +\bm{z_2}^{\top}\bm{\tau^{d}}\\
    &\le -\frac{\partial \bm{\Phi}}{\partial \bm{z_1}}^{\top} \frac{\partial \bm{\Phi}}{\partial \bm{z_1}}-\bm{z_2}^{\top}\bm{D}(\bm{z_2})\bm{z_2} +\frac12 \bm{z_2}^{\top}\bm{z_2}+\frac12 \|\bm{\tau^{d}}\|^2.
\end{align*}
Since $\bm{D}(\bm{z_2})\ge \frac12 \bm{I}$, we have
$    \bm{\dot{V}}\le -\frac{\partial \bm{\Phi}}{\partial \bm{z_1}}^{\top} \frac{\partial \bm{\Phi}}{\partial \bm{z_1}}+\frac12 d.$
Then $\bm{z_1}$ will converge to the set $\{\bm{z_1}|\frac{\partial \bm{\Phi}}{\partial \bm{z_1}}^{\top} \frac{\partial \bm{\Phi}}{\partial \bm{z_1}}\le \frac12 d \}$.
\end{proof}

We can achieve $\bm{D}(\bm{z_2})\ge \frac12 \bm{I}$ by adjusting the hyperparameter $m$.
In Theorem \ref{Th. disturbance}, the bound of the tracking error is given as a function of the Jacobian of $\bm{\Phi}$.
In the following, we will provide an explicit tracking error bound assuming structures of $\bm{\Phi}$. 

\begin{theorem}
  Consider the system~\eqref{eq.manipulator} with a bounded disturbance $\|\bm{\tau^{d}}\|^2\le d$.
  If $\bm{D}(\bm{z_2})\ge \frac12 \bm{I}$ and $ \left .\frac{ \partial^2\bm{\Phi}}{\partial \bm{z_1}^2}\right|_{\bm{z_1}=0} \ge \alpha \bm{I} $, then the tracking error under the controller~\eqref{eq.trcaking controller}  will converge to the set $\{ \bm{z_1}| \|\bm{z_1}\|^2\le \frac12 \frac{1}{\alpha^2}d\}$. 
\label{Th. disturbance_regularizer}
\end{theorem}
\begin{proof}The Taylor expansion of $\frac{\partial \bm{\Phi}}{\partial \bm{z_1}}$ at $\bm{z_1}=\bm{0}$ is 
\begin{equation}
    \begin{aligned}
    \frac{\partial \bm{\Phi}}{\partial \bm{z_1}}=\left .\frac{ \partial^2\bm{\Phi}}{\partial \bm{z_1}^2}\right|_{\bm{z_1}=\bm{0}}  \bm{z_1} +\delta
\end{aligned} \label{eq.taylor expansion}
\end{equation}
where $\delta$ contains high-order terms.
We can approximate $\frac{\partial \bm{\Phi}}{\partial \bm{z_1}}$ with $\left .\frac{ \partial^2\bm{\Phi}}{\partial \bm{z_1}^2}\right|_{\bm{z_1}=\bm{0}}  \bm{z_1}$.
Then from $\frac{\partial \bm{\Phi}}{\partial \bm{z_1}}^{\top} \frac{\partial \bm{\Phi}}{\partial \bm{z_1}}\le \frac12 d$, we have $\bm{z_1}^{\top} \left( \left .\frac{ \partial^2\bm{\Phi}}{\partial \bm{z_1}^2}\right|_{\bm{z_1}=\bm{0}}\right)^2 \bm{z_1} \le \frac12 d$. 
Since $\left .\frac{ \partial^2\bm{\Phi}}{\partial \bm{z_1}^2}\right|_{\bm{z_1}=\bm{0}}\ge \alpha \bm{I}$, therefore $\alpha^2 \|\bm{z_1}\|^2 \le \frac12 d$, that is, $\|\bm{z_1}\|^2\le \frac12 \frac{1}{\alpha^2}d$. 
\end{proof}
    
These analyses imply that we can constrain the Hessian matrix of $\bm{\Phi}$ at $\bm{z_1}=\bm{0}$ to improve the tracking error performance under disturbances.
This can be achieved by adding a regularizer $\text{ReLu}\left(\lambda_{\max}\left(\alpha \bm{I} - \left .\frac{ \partial^2\bm{\Phi}}{\partial \bm{z_1}^2}\right|_{\bm{z_1}=\bm{0}} \right) \right )$ during training or by letting the hyperparameter $\bm{S}$ satisfy $\bm{S} \geq \alpha\bm{I}$ in $\bm{\Phi}(\bm{z_1};\bm{\theta_1},\bm{S})$. 

\subsection{Learning Optimization Formulation}\label{sec.learningOptFormulation}

Since the controller \eqref{eq.trcaking controller} is stable for all compatible DNN parameters, we can further optimize the performance by optimizing the DNN parameters.
The optimization we are solving is the following. 

\begin{equation}
    \begin{aligned}
        \min_{\bm{\theta_1},\bm{\theta_2}} \quad & \int_{t=0}^{T} l_t(\bm{z_1}, \bm{u})dt\\
        &\qquad +\text{ReLu}\left(\lambda_{\max}\left(\alpha \bm{I} - \left .\frac{ \partial^2\bm{\Phi}}{\partial \bm{z_1}^2}\right|_{\bm{z_1}=\bm{0}} \right) \right )  \\
     \mathrm{s.t.} \quad&\text{dynamics } \bm{M}(\bm{q})\bm{\ddot{q}}+\bm{C}(\bm{q},\bm{\dot{q}})\bm{\dot{q}}+\bm{G}(\bm{q})=\bm{u}, \\
      &\text{errors } \eqref{eq.errors}, \quad
      \text{controller }\eqref{eq.trcaking controller},\\
    & \text{initial state } \bm{q}(0),\bm{\dot{q}}(0),
    \end{aligned}\label{eq.optimization}
\end{equation}
where $l_t$ denotes the stage cost at time $t$.
The above optimization problem can be solved by discretizing the dynamics and the cost function first and then numerically solving the discrete-time counterpart. 

\section{LNNs based Model Approximation and Controller Design}\label{sec.LNN based}
In Section~\ref{sec.main result}, the inertia matrix $\bm{M}$, the Coriolis matrix $\bm{C}$, and the gravitational force vector $\bm{G}$ of the model are used to design the NBS tracking controller.
However, these terms are difficult to obtain for complex systems.
In this section, we use LNN $\bm{\mathcal{L}}(\bm{q},\bm{\dot{q}};\bm{\gamma})$ to learn the Lagrangian of the system from data and then construct approximated inertia matrix, Coriolis matrix and gravitational force vectors to design the controller.
To ensure a better approximation of the Lagrangian function, we proposed a modified LNN structure, which can guarantee that Property \ref{property. M definit} and Property \ref{property. M-2C} hold for the approximated inertia matrix and the Coriolis matrix.

\subsection{Modified LNN for Learning Dynamics}
In LNNs,  FCNNs are used to represent the Lagrangian function.
In doing so, the approximated inertia matrix $\bm{\hat{M}}=\frac{\partial^{2}\mathcal{L}(\bm{q},\bm{\dot{q}};\bm{\gamma^*})}{\partial \bm{\dot{q}}\partial\bm{\dot{q}}}$ and the Coriolis matrix $\bm{\hat{C}}=\bm{\dot{\hat{M}}}(\bm{q})-\frac{1}{2}\bm{\dot{\hat{M}}}(\bm{q})^{\top}$may not satisfy the Property \ref{property. M definit} and Property \ref{property. M-2C}.
To solve this problem, we use an ICNN $\mathcal{L}_T(\bm{q},\bm{\dot{q}};\bm{\gamma_1})$ with the parameter $\bm{\gamma_1}$ to approximate the kinetic energy $T(\bm{q},\bm{\dot{q}})$ in $\mathcal{L}$, which is convex w.r.t. $\bm{\dot{q}}$.
Moreover, we use an FCNN $\mathcal{L}_V(\bm{q};\bm{\gamma_2})$ with the parameter $\bm{\gamma_2}$ and input $\bm{q}$ to learn $V(\bm{q})$ in $\mathcal{L}$.
Then we construct the Lagrangian neural network as follows. 
\begin{align*}
    \mathcal{L}(\bm{q},\bm{\dot{q}};\bm{\gamma})=\mathcal{L}_T(\bm{q},\bm{\dot{q}};\bm{\gamma_1})-\mathcal{L}_V(\bm{q};\bm{\gamma_2}),
\end{align*}
where $\bm{\gamma}=\{\bm{\gamma_1},\bm{\gamma_2}\}$,  the activation function we choose is softplus.
Through this design, we can ensure that $\bm{\hat{M}}=\frac{\partial^{2}\mathcal{L}(\bm{q},\bm{\dot{q}};\bm{\gamma^*})}{\partial \bm{\dot{q}}\partial\bm{\dot{q}}}=\frac{\partial^{2}\mathcal{L}_T(\bm{q},\bm{\dot{q}};\bm{\gamma^*})}{\partial \bm{\dot{q}}\partial\bm{\dot{q}}}$ is positive definite and $\bm{\dot{\hat{M}}}-2\bm{\hat{C}}$ is skew-symmetric.

Using LNNs to approximate the dynamic system model may introduce approximation errors.
With $\bm{\hat{M}}, \bm{\hat{C}}, \bm{\hat{G}}$, the system dynamics can be modeled as 
    \begin{equation}
    \begin{aligned}
\bm{\hat{M}}\bm{\ddot{q}}+\bm{\hat{C}}\bm{\dot{q}}+\bm{\hat{G}}+\delta(\bm{q},\bm{\dot{q}})=\bm{u}+\bm{\tau^d},
    \end{aligned} \label{eq.LNNs-based uncertain system}
\end{equation}
where $\delta(\bm{q},\bm{\dot{q}})=-\bm{\hat{M}}\bm{M}^{-1}\{\bm{u}+\bm{\tau^d}-\bm{C}\bm{\dot{q}}-\bm{G}\}+\{\bm{u}+\bm{\tau^d}-\bm{\hat{C}}\bm{\dot{q}}-\bm{\hat{G}}\}$ is the model mismatch.

Since $\delta$ is a function of $\bm{q}$ and $\bm{\dot{q}}$, we can employ the first-order Taylor expansion and assume that the uncertain term satisfies the inequality
\begin{equation}
    \begin{aligned}
        \|\delta(\bm{q},\bm{\dot{q}})\| \leq a\|\bm{q}\|+b\|\bm{\dot{q}}\|+c,
    \end{aligned} \label{eq.LNNs-based uncertain system uncertain constraints}
\end{equation}
where $a,b,c$ are positive constants. 

\subsection{Controller Design and Performance Analysis}
With the approximate model parameters, the NBS tracking controller~\eqref{eq.trcaking controller} is modified to
\begin{equation}
    \begin{aligned}
\bm{u}=&\bm{\hat{G}}+\bm{\hat{M}}\bm{\dot{\phi}}+\bm{\hat{C}}\bm{\phi}- \frac{\partial \bm{\Phi}(\bm{z_1};\bm{\theta_1},\bm{S})}{\partial \bm{z_1}} -\bm{D}(\bm{z_2};\bm{\theta_2})\bm{z_2},\\
\bm{\phi}=&\bm{\dot{q}^{d}}- \frac{\partial \bm{\Phi}(\bm{z_1};\bm{\theta_1},\bm{S})}{\partial \bm{z_1}}.
    \end{aligned}\label{eq.LNNs controller}
\end{equation}
In the following theorem, we demonstrate that the modified NBS tracking controller can achieve bounded tracking error in the presence of modeling uncertainties.

\begin{theorem}
  Consider the system~\eqref{eq.LNNs-based uncertain system} with a bounded disturbance $\|\bm{\tau^{d}}\|^2\le d$,  and the model uncertainty satisfies~\eqref{eq.LNNs-based uncertain system uncertain constraints}. If the controller is designed as~\eqref{eq.LNNs controller}, with $\bm{D} \geq (b+\frac{b^2}{2}+\frac{a}{2}+1)\bm{I},\left .\frac{ \partial^2\bm{\Phi}}{\partial \bm{z_1}^2}\right|_{\bm{z_1}=0} \ge \alpha \bm{I} $ and $\alpha>\sqrt{a}$,  then the tracking error $z_1$ will converge to the set $\{\bm{z}_1|\|\bm{z_1}\|^2\leq\frac{k^2+d}{\alpha^2-a}\}$, where $k= \max(c+a\|\bm{q^d}\|+b\|\bm{\dot{q}^d}\|)$.
    \label{Th. LNNs_uncertainty_regularizer}
\end{theorem}
\begin{proof}
    From \eqref{eq.errors} we have
\begin{align*}
    \|\bm{q}\|\leq \|\bm{z_1}\| +\|\bm{q^d}\|, \quad 
    \|\bm{\dot{q}}\| \leq \|\bm{z_2}\|+\|\bm{\dot{q}^d}\|+\|\frac{\partial \bm{\Phi}(\bm{z_1})}{\partial \bm{z_1}}\|.
\end{align*}
    According to \eqref{eq.LNNs-based uncertain system uncertain constraints}, one can obtain
        $\|\delta(\bm{q},\bm{\dot{q}})\| \leq a\|\bm{z_1}\|+b\|\bm{z_2}\|+b\|\frac{\partial \bm{\Phi}(\bm{z_1})}{\partial \bm{z_1}}\|+c+a\|\bm{q^d}\|+b\|\bm{\dot{q}^d}\|.$
Since $\bm{q^d}, \bm{\dot{q}^d}$ are bounded reference states, we can define $k= \max(c+a\|\bm{q^d}\|+b\|\bm{\dot{q}^d}\|)$. Then we have the following.
    \begin{align*}
        &\|\delta(\bm{q},\bm{\dot{q}})\| \leq a\|\bm{z_1}\|+b\|\bm{z_2}\|+b\|\frac{\partial \bm{\Phi}(\bm{z_1})}{\partial \bm{z_1}}\|+k.
    \end{align*}
Consider the following Lyapunov candidate
    \begin{align*}
        V(\bm{z_1},\bm{z_2})=\Phi(\bm{z_1})+\frac{1}{2}\bm{z_2}^{\top}\bm{\hat{M}}\bm{z_2}
    \end{align*}
The time derivative of $V$ is 
    \begin{align*}
        \dot{V}=&\bm{z_2}^{\top}\left( -\frac{\partial \bm{\Phi}}{\partial \bm{z_1}}-\bm{D}\bm{z_2} +\bm{\tau^d}-\delta \right)+\bm{\dot{z}_1}^{\top} \frac{\partial \bm{\Phi}}{\partial \bm{z_1}}\\
        =&-\frac{\partial \bm{\Phi}}{\partial \bm{z_1}}^{\top} \frac{\partial \bm{\Phi}}{\partial \bm{z_1}}-\bm{z_2}^{\top}\bm{D}\bm{z_2}+\bm{z_2}^{\top}(\bm{\tau^d}-\delta)\\
         \leq& -\frac{\partial \bm{\Phi}}{\partial \bm{z_1}}^{\top} \frac{\partial \bm{\Phi}}{\partial \bm{z_1}} -\lambda_{\min}(\bm{D})\|\bm{z_2}\|^2 +\|\bm{\tau^d}\|\|\bm{z_2}\|+b\|\bm{z_2}\|^2\\
         &+a\|\bm{z_1}\|\|\bm{z_2}\|+k\|\bm{z_2}\|+b\|\bm{z_2}\|\|\frac{\partial \bm{\Phi}}{\partial \bm{z_1}}\| \\
        \leq &-\frac{\partial \bm{\Phi}}{\partial \bm{z_1}}^{\top} \frac{\partial \bm{\Phi}}{\partial \bm{z_1}}+\frac{1}{2}\frac{\partial \bm{\Phi}}{\partial \bm{z_1}}^{\top} \frac{\partial \bm{\Phi}}{\partial \bm{z_1}}+\frac{a}{2}\|\bm{z_1}\|^2+\frac{k^2}{2}+\frac{1}{2}\|\bm{\tau^d}\|^2 \\
        &-(\lambda_{\min}(\bm{D})-b-\frac{b^2}{2}-\frac{a}{2}-1)\|\bm{z_2}\|^2 \\
        =&-\frac{1}{2}\frac{\partial \bm{\Phi}}{\partial \bm{z_1}}^{\top} \frac{\partial \bm{\Phi}}{\partial \bm{z_1}}+\frac{a}{2}\|\bm{z_1}\|^2+\frac{k^2+d}{2}\\
        &-(\lambda_{\min}(\bm{D})-b-\frac{b^2}{2}-\frac{a}{2}-1)\|\bm{z_2}\|^2
    \end{align*}
Since $\bm{D}\ge(b+\frac{b^2}{2}+\frac{a}{2}+1)\bm{I}$, we have 
    $ h=\lambda_{\min}(\bm{D})-b-\frac{b^2}{2}-\frac{a}{2}-1\geq 0.$
Therefore,
 \begin{align*}
     \dot{V} &\leq -\frac{1}{2}\frac{\partial \bm{\Phi}}{\partial \bm{z_1}}^{\top} \frac{\partial \bm{\Phi}}{\partial \bm{z_1}}+\frac{a}{2}\|\bm{z_1}\|^2+\frac{k^2+d}{2}.
 \end{align*}



  Similar to the proof of Theorem~\ref{Th. disturbance_regularizer}, if  $\left .\frac{ \partial^2\bm{\Phi}}{\partial \bm{z_1}^2}\right|_{\bm{z_1}=\bm{0}} \ge \alpha \bm{I}$, we further have
    \begin{align*}
     \dot{V} \leq& -\frac{\alpha}{2}\|\bm{z_1}\|^2+\frac{a^2}{2}\|\bm{z_1}\|^2+\frac{k^2+d}{2} \\
     =&-\frac{\alpha^2-a}{2}\|\bm{z_1}\|^2+\frac{k^2+d}{2}.
 \end{align*}
    Then we can conclude that the tracking error will converge to the set $\{\bm{z_1}|\|\bm{z_1}\|^2\leq\frac{k^2+d}{\alpha^2-a}\}$.
\end{proof}

In the preceding proof, we show that the NBS tracking controller can achieve bounded tracking errors and the explicit tracking error bounds can be further obtained by carefully designing $\bm{\Phi}$.
In view of Theorem~\ref{Th. LNNs_uncertainty_regularizer}, we can set the hyperparameter $m\geq\sqrt{b+\frac{b^2}{2}+\frac{a}{2}+1}$ in $\bm{D}$ to guarantee $\bm{D} \geq (b+\frac{b^2}{2}+\frac{a}{2}+1)\bm{I}$.
Moreover, we can set the hyperparameter $\bm{S} \geq \alpha\bm{I}$ in $\bm{\Phi}$ or add the regularizer $\text{ReLu}\left(\lambda_{\max}\left(\alpha \bm{I} - \left .\frac{ \partial^2\bm{\Phi}}{\partial \bm{z_1}^2}\right|_{\bm{z_1}=0} \right) \right )$ during training to ensure that $\left .\frac{ \partial^2\bm{\Phi}}{\partial \bm{z_1}^2}\right|_{\bm{z_1}=\bm{0}} \ge \alpha \bm{I}$.
Similarly as in Section~\ref{sec.learningOptFormulation}, we can further optimize the neural network parameters to achieve better control performance.


\section{Simulations}\label{sec.experiment}
In this section, in order to validate our method, we apply the NBS controller to the tracking control of manipulators.
We first perform the sin and cos signal tracking task for the two-link planar robot arm with known model information.
Our experiments encompass scenarios with and without disturbances and verify the controller performance.
Subsequently, using MuJoCo~\cite{todorov2012mujoco}, we extend our experimentation to a three-link system, employing LNN to learn the dynamic model of the system and execute trajectory tracking control.
We use pytorch with Adam ~\cite{DBLP:journals/corr/KingmaB14} optimizer throughout the experiments.
The code is available from: \url{https://github.com/jiajqian/Neural-Backstepping-tracking-controller}.

\subsection{Two-Link Planar Robot Arm with Known Model Information}
We consider the two-link planar robot arm model as shown in Fig.~\ref{two-link planar robot arm model}, where the link masses are concentrated at the ends of the links ~\cite{lewisRobotManipulatorControl2004}. 
The two-link planar robot arm has $2$ control inputs $[u_1,u_2]$ denoting the torque applied to each link and $4$ state variables $[\beta_1,\beta_2,\bm{\dot{\beta_1}},\bm{\dot{\beta_2}}]$, representing the angle and angular velocity of the links.
Each link has mass $m_i=1 \mathrm{kg}$ and length $l_i=1 \mathrm{m}$, where $i = 1,2$.
    The dynamic model of the two-link planar robot arm can be described by~\eqref{eq.Lagrangian} with $\bm{q}=[\beta_1, \beta_2], \bm{\dot{q}}=[\bm{\dot{\beta}_1}, \bm{\dot{\beta}_2}]$,
\begin{align*}
  T(\bm{\dot{q}})=&\frac{1}{2}(3+2\cos{\beta_2})\bm{\dot{\beta_1}}^2+(1+\cos{\beta_2})\bm{\dot{\beta_1}}\bm{\dot{\beta_2}}
  +\frac{1}{2}\bm{\dot{\beta_2}}^2,\\
      V(\bm{q})=&2g\sin{\beta_1}+g\sin(\beta_1+\beta_2),
\end{align*}
where $g=9.8\mathrm{N/kg}$. The control target is to let $\beta_1$ track the signal $sin(0.1t)$ and $\beta_2$ track the signal $cos(0.1t)$. The desired trajectory is $\bm{q^{d}}=[\beta_1^{d}, \beta_2^{d}]=[sin(0.1t),cos(0.1t)], \bm{\dot{q}^{d}}=[\bm{\dot{\beta}_1^{d}}, \bm{\dot{\beta}_2^{d}}]=[cos(0.1t),-sin(0.1t)]$
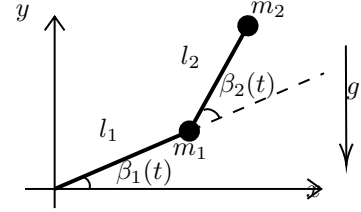
\begin{figure}[htpb]
      \centering

\tikzset{every picture/.style={line width=0.75pt}} 

\begin{tikzpicture}[x=0.75pt,y=0.75pt,yscale=-1,xscale=1]

\draw  (97.59,266.35) -- (246.98,266.35)(112.53,179.33) -- (112.53,276.01) (239.98,261.35) -- (246.98,266.35) -- (239.98,271.35) (107.53,186.33) -- (112.53,179.33) -- (117.53,186.33)  ;
\draw [line width=1.5]    (112.53,266.35) -- (180.47,237.14) ;
\draw [shift={(180.47,237.14)}, rotate = 336.74] [color={rgb, 255:red, 0; green, 0; blue, 0 }  ][fill={rgb, 255:red, 0; green, 0; blue, 0 }  ][line width=1.5]      (0, 0) circle [x radius= 4.36, y radius= 4.36]   ;
\draw [line width=1.5]    (180.47,237.14) -- (210,184) ;
\draw [shift={(210,184)}, rotate = 299.06] [color={rgb, 255:red, 0; green, 0; blue, 0 }  ][fill={rgb, 255:red, 0; green, 0; blue, 0 }  ][line width=1.5]      (0, 0) circle [x radius= 4.36, y radius= 4.36]   ;
\draw  [draw opacity=0] (187.72,226.62) .. controls (189.09,225.56) and (191.72,226.06) .. (193.74,227.8) .. controls (194.67,228.61) and (195.31,229.53) .. (195.63,230.43) -- (191.37,230.06) -- cycle ; \draw   (187.72,226.62) .. controls (189.09,225.56) and (191.72,226.06) .. (193.74,227.8) .. controls (194.67,228.61) and (195.31,229.53) .. (195.63,230.43) ;  
\draw  [draw opacity=0] (126.39,260.23) .. controls (126.48,260.25) and (126.57,260.27) .. (126.66,260.29) .. controls (129.34,260.95) and (130.88,263.57) .. (130.11,266.15) .. controls (130.11,266.16) and (130.1,266.18) .. (130.1,266.19) -- (125.26,264.95) -- cycle ; \draw   (126.39,260.23) .. controls (126.48,260.25) and (126.57,260.27) .. (126.66,260.29) .. controls (129.34,260.95) and (130.88,263.57) .. (130.11,266.15) .. controls (130.11,266.16) and (130.1,266.18) .. (130.1,266.19) ;  
\draw [line width=0.75]  [dash pattern={on 4.5pt off 4.5pt}]  (180.47,237.14) -- (248.41,207.94) ;
\draw    (259.41,193.9) -- (259.64,252.48) ;
\draw [shift={(259.65,254.48)}, rotate = 269.77] [color={rgb, 255:red, 0; green, 0; blue, 0 }  ][line width=0.75]    (10.93,-3.29) .. controls (6.95,-1.4) and (3.31,-0.3) .. (0,0) .. controls (3.31,0.3) and (6.95,1.4) .. (10.93,3.29)   ;

\draw (237.91,263.13) node [anchor=north west][inner sep=0.75pt]    {$x$};
\draw (91.45,171.88) node [anchor=north west][inner sep=0.75pt]    {$y$};
\draw (194.01,206) node [anchor=north west][inner sep=0.75pt]    {$\beta _{2}( t)$};
\draw (142,249.94) node [anchor=north west][inner sep=0.75pt]    {$\beta _{1}( t)$};
\draw (133.46,229.51) node [anchor=north west][inner sep=0.75pt]    {$l_{1}$};
\draw (174.76,192.46) node [anchor=north west][inner sep=0.75pt]    {$l_{2}$};
\draw (170.72,242) node [anchor=north west][inner sep=0.75pt]    {$m_{1}$};
\draw (210,170.49) node [anchor=north west][inner sep=0.75pt]    {$m_{2}$};
\draw (258.25,211.12) node [anchor=north west][inner sep=0.75pt]    {$g$};

\end{tikzpicture}

      \caption{The model of a 2-link planar robot arm}

      \label{two-link planar robot arm model}
   \end{figure}

We define the errors as ~\eqref{eq.errors}. Therefore, the control target can be expressed as stabilizing the errors, that is, ensuring $\underset{t\rightarrow \infty}{\lim}\bm{z_1}=0$.
   
   The NBS tracking controller is designed as ~\eqref{eq.trcaking controller}. 
   In the controller, $\bm{\psi}$ has $3$ hidden layers and each hidden layer has $32$ neurons, $\bm{S}=\bm{I}$.
Moreover, each FCNN in $\bm{D}$ has $2$ hidden layers and each hidden layer has $32$ neurons, $m=0.001$.
 
To demonstrate the unconditional stability of the NBS tracking controller, we conduct tracking experiments without any prior training in the parameters of the neural network. The simulations are performed with a time step of $0.01$ seconds. In addition to the state $[0,0,0,0]$, we also randomly selected three initial states.  Fig.~\ref{fig:simulate without training} depicts the joint angles of the robot $\bm{q}$ and the tracking errors $\bm{z_1}$, representing the deviation between $\bm{q}$ and the desired trajectory $\bm{q^{d}}$  start from different initial states.
Remarkably, the NBS tracking controller can successfully track the sin and cos signals, respectively, even without prior training.
\begin{figure}
    \centering
    \includegraphics[scale=0.18]{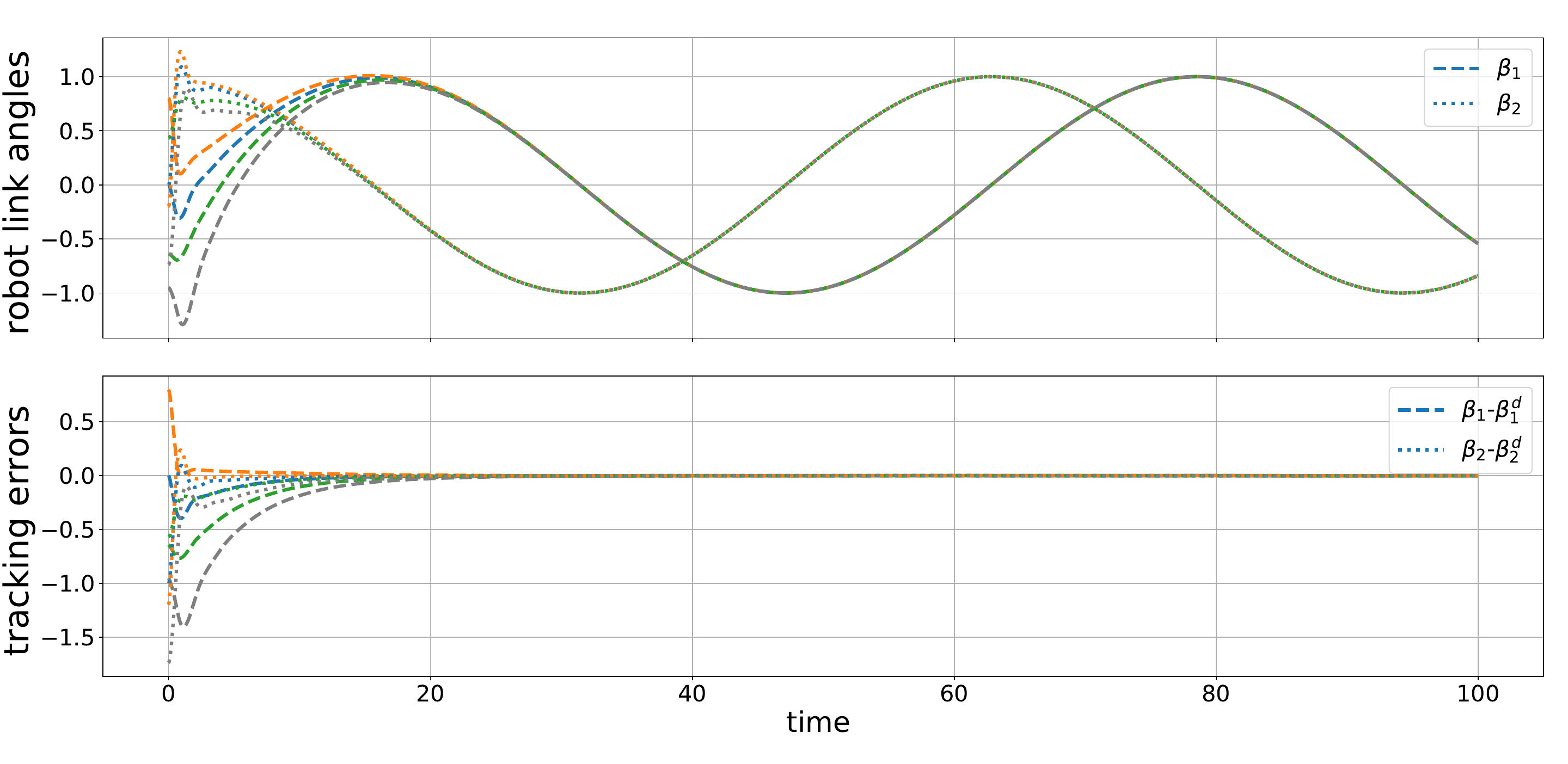}
    \caption{The time evolution of angles of the two-link planar robot arm using the NBS tracking controller without training. Different initial states and state variables are marked by different line colors and styles respectively.}
    \label{fig:simulate without training}
\end{figure}

We can improve the performance of the NBS tracking controller by solving the optimization problem ~\eqref{eq.optimization}  to achieve better performance.
We discrete the optimization problem ~\eqref{eq.optimization} with a step size of $0.01s$, the time horizon $T=1s$, start from the initial state $[0,0,0,0]$, and choose the stage cost as $l_t=\bm{z_1}^{\top}\bm{z_1}$. We use a decaying learning rate that starts at $1e^{-3}$ for $200$ epochs. Fig.~\ref{fig:simulate after training} shows the trajectory and tracking error of the angles using the trained controller  and the tranditional PID controller as a baseline.
As depicted in Fig. ~\ref{fig:simulate after training}, it is evident that the NBS tracking controller effectively steers the angles, rapidly aligning them with the desired trajectory. This outcome signifies a noticeable improvement in performance. And we can observe that the NBS tracking controller has better performance than the traditional PID controller.

\begin{figure}
    \centering
    \includegraphics[scale=0.18]{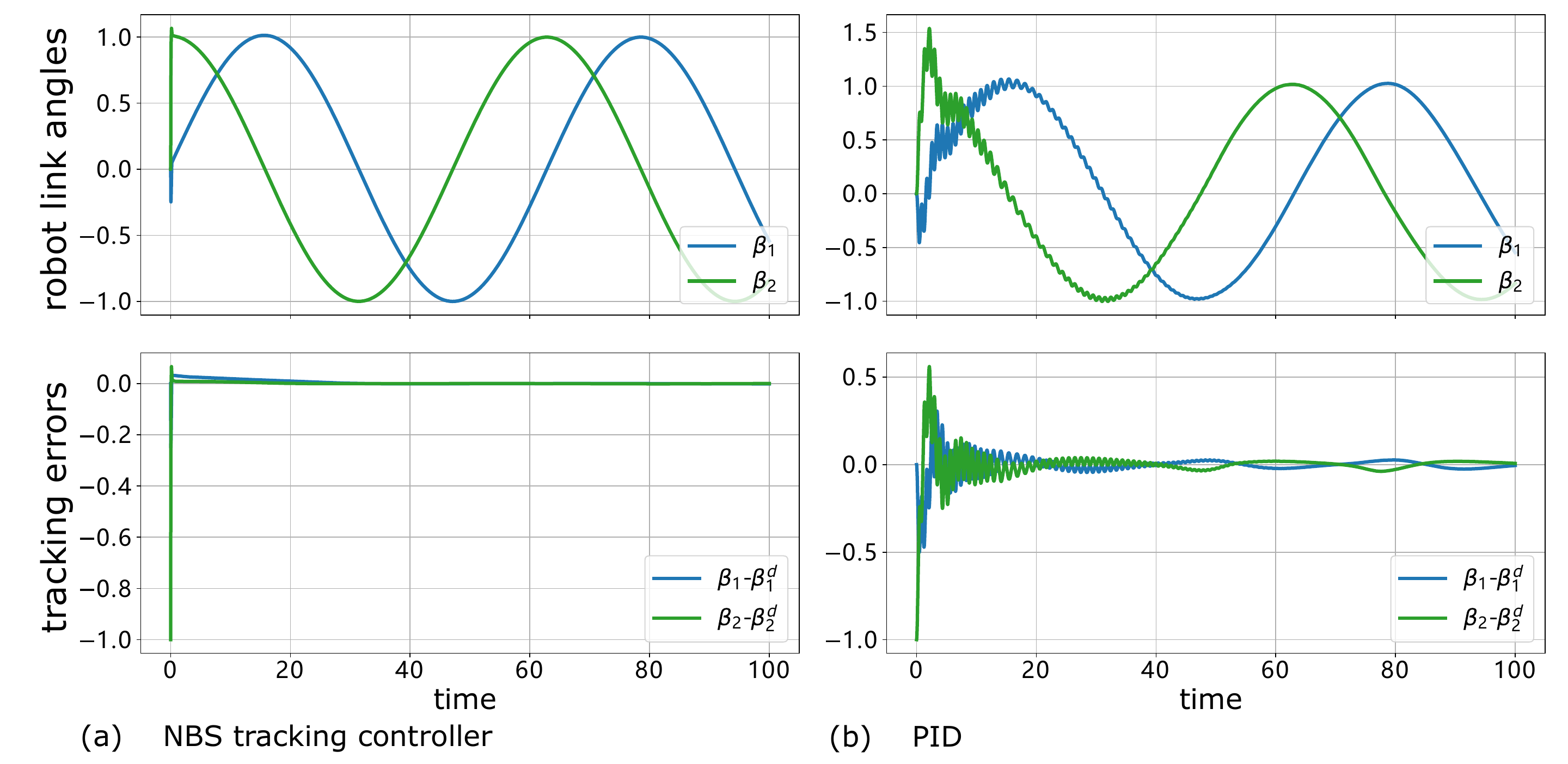}
    \caption{The time evolution of angles of the two-link planar robot arm (a) using the NBS tracking controller after training, (b) using the PID controller.}
    \label{fig:simulate after training}
\end{figure}

Traditional PID controllers cannot effectively deal with the presence of disturbances. To demonstrate that the NBS tracking controller can guarantee a bounded tracking error in the presence of disturbances,  we consider the constant disturbance case $\bm{\tau^{d}}=[1.0,1.0]$. We set the hyperparameter $m$ in $\bm{D}$ to $1.0$ to ensure $\bm{D}(\bm{z_2})\ge \frac12 \bm{I}$.
We choose the stage cost as $l_t=\bm{z_1}^{\top}\bm{z_1} $, discrete the optimization problem ~\eqref{eq.optimization} with a step size of $0.01s$ and the time horizon $T=1s$.
We solve the optimization problem ~\eqref{eq.optimization} to get $\theta_1^*,\theta_2^*$.
To investigate the impact of the parameter $\alpha$ of the regularization term on the stable tracking error bound, we uniformly sample $40$ values for $\alpha$ from $[0,2]$.
We use a learning rate $1e^{-3}$ for each $\alpha$ over a span of $200$ epochs.  According to Thereom \ref{Th. disturbance_regularizer}, the tracking error should converge to $\|\bm{z_1}\|^2\le \frac{1}{2\alpha^2}\times2$.
Fig.~\ref{fig:error to alpha} shows that the NBS tracking controller can be trained to ensure a small steady-state tracking error, and the error is bounded in the presence of disturbances.
It can be seen that increasing the value $\alpha$ can reduce the maximum steady-state error.
However, it should be noted that $\alpha$ should not be set too large; otherwise, this will make learning the optimal controller difficult.

\begin{figure}[!t]
    \centering
    \includegraphics[scale=0.17]{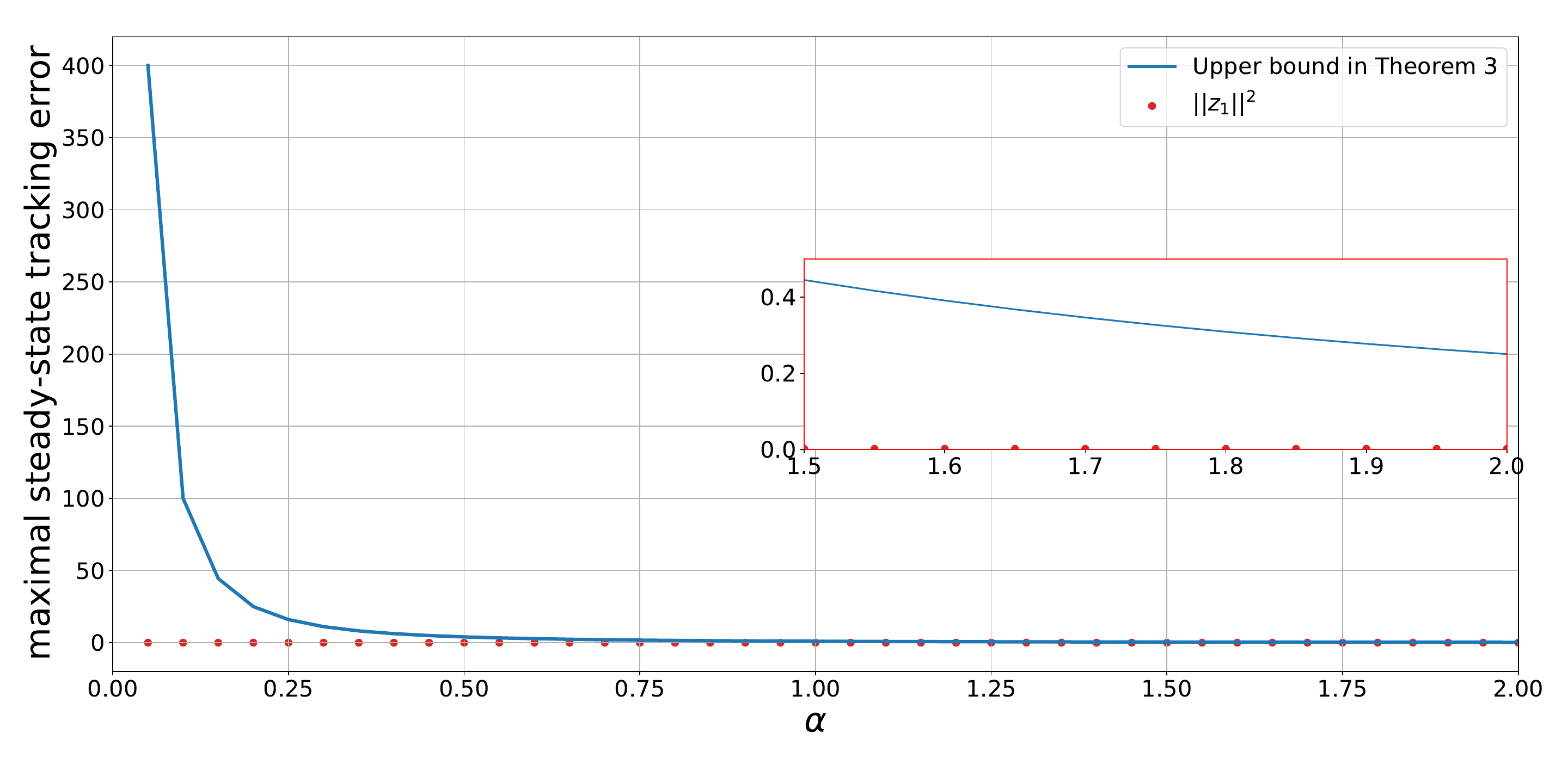}
    \caption{The steady state tracking error under different $\alpha$ and the corresponding upper bounds in Theorem \ref{Th. disturbance_regularizer}}
    \label{fig:error to alpha}
\end{figure}

The specific performances of different simulations are shown in Table \ref{tab.performance}. We simulate 100s, select the maximum value of $\|\bm{z_1}\|^2$ in the last $30s$ as the steady state tracking error, and select the first time when $\|\bm{z_1}\|^2$ remains less than $0.01$ as the convergence time. From Table \ref{tab.performance}, we can see more intuitively that our controller has higher tracking accuracy and faster convergence speed than the traditional PID controller. Furthermore, by optimizing the DNN parameters our controller can achieve better performance, especially for the convergence time. And our controller can ensure stability and maintain a small tracking error in the presence of disturbances.

    \begin{table}
        \begin{center}
        \caption{The performance of different controllers, when the system start from the initial states $[0,0,0,0]$}
        \label{tab.performance}
        \begin{tabular}{|>{\centering\arraybackslash}m{3.2cm}|>{\centering\arraybackslash}m{2cm}|>{\centering\arraybackslash}m{2cm}|}
        \hline
              controllers&  steady state tracking error&convergence time\\
                    \hline
     PID controller& 2.06$e^{-3}$&15.54s\\\hline
              NBS tracking controller (without training)& 1.30$e^{-6}$&5.72s\\
              \hline
     NBS tracking controller    
    (after training)& 1.22$e^{-6}$&0.13s\\ \hline
     NBS tracking controller ($\alpha=1$, $\bm{\tau^d}=[1.0,1.0]$ )& 1.25$e^{-3}$&0.13s\\\hline
     
     \end{tabular}
       
        \end{center}
         
    \end{table}

\subsection{Three-link Planar Robot Arm with Unknown Model Information}
The mathematical model of a three-link planar robot arm is difficult to obtain.
Therefore, we propose to use the MuJoCo physics simulator~\cite{todorov2012mujoco} to sample data and then learn the Lagrangian function of the three-link planar robot arm, which gives an approximated model of the three-link planar robot arm.
The model we built in MuJoCo 
has $3$ control inputs $[u_1,u_2,u_3]$ denoting the torque applied to each link, and $6$ state variables $[\beta_1,\beta_2,\beta_3,\bm{\dot{\beta_1}},\bm{\dot{\beta_2}},\bm{\dot{\beta_3}}]$, representing links' angle and angular velocity. The desired tracking trajectory is $[\beta_1^{d}=sin(0.1t),\beta_2^{d}=cos(0.1t),\beta_3^{d}=sin(0.1t)]$ . 


When building the LNN $\mathcal{L}$, $\mathcal{L}_T$ is a PICNN in the form of \eqref{eq.standard PICNN}.
$\mathcal{L}_T$ and $\mathcal{L}_V$ have $3$ hidden layers, each layer has $32$ neurons, and the activation function is softplus.
First, we let the MuJoCo simulator run without imposing any control input on the robot arm, and obtain the state information of the model, gathering the dataset for training $\mathcal{L}$.
We choose the simulation step size to be $0.001s$, the initial state to be $[0,0,0,0,0,0]$, the control inputs $\bm{u}=0$, and sample $10,000$ points.
Then we solve \eqref{eq.LNNs optimization} and obtain $\bm{\gamma^*}$.
We train $\mathcal{L}$ for $200$ epochs, a batch size of $10$ and a decaying learning rate start at $1e^{-3}$.

After training $\mathcal{L}$, we use the LNN-based NBS tracking controller as \eqref{eq.LNNs controller}.
We discrete the optimization problem \eqref{eq.LNNs optimization} with a step size of $0.01s$, the time horizon $T=1s$ and choose the stage cost as $l_t=\bm{z_1}^{\top}\bm{z_1}$.
We still use a decaying learning rate starting at $1e^{-3}$ for $200$ epochs. 
Fig. \ref{fig:mujoco LNNs trajectory} shows the angular trajectory and tracking errors.
The NBS tracking controller effectively achieves high-precision  tracking, maintaining a steady-state tracking error with $\bm\|{z_1}\|^2 \leq 1.5e^{-3}$.
This experiment further demonstrates the efficacy of the LNNs-based NBS tracking controller, particularly in scenarios where the system model remains unknown or difficult to obtain. 

\begin{figure}
\centering
    \includegraphics[scale=0.18]{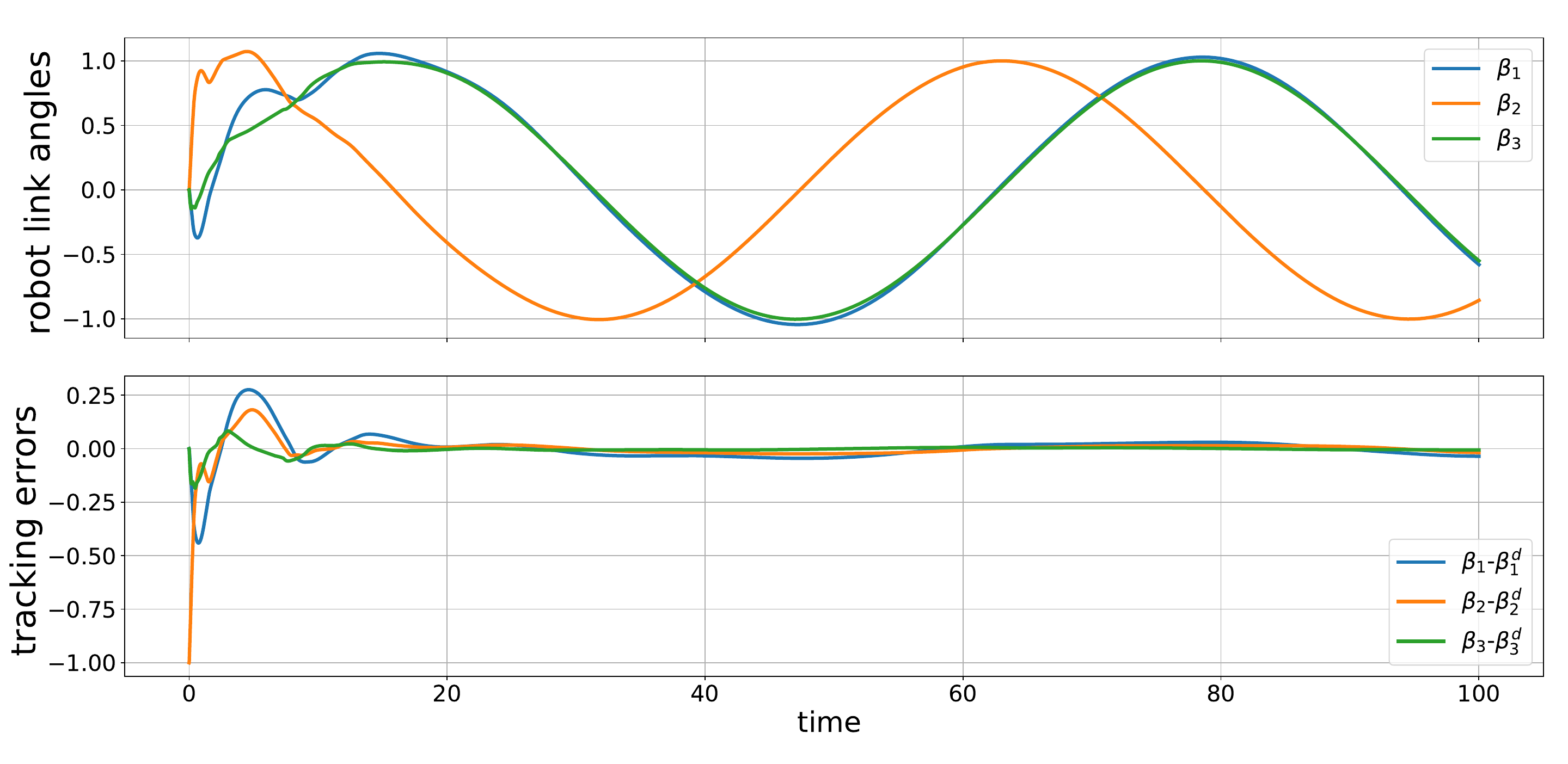}
    \caption{The time evolution of angles of the three-link planar robot arm using LNNs-based NBS tracking controller under MuJoCo simulator.}
    \label{fig:mujoco LNNs trajectory}
\end{figure}

\section{Conclusions}\label{sec.conclusion}
In this paper, we propose a structured DNN controller based on backstepping methods.
With properly designed DNN structures, the controller has unconditional stability guarantees.
In addition, its parameters can be optimized to achieve better performance.
We further prove that the tracking error is bounded in the presence of disturbances.
When the model information is unknown, we use ICNN to improve the LNN structure and use the improved LNNs to learn system dynamics.
The controller is then designed on the basis of the learned system dynamics.
We can also prove that the tracking error is bounded in the presence of model uncertainties and external disturbances.
In the future, we plan to generalize the method to general nonlinear systems. 


\bibliographystyle{IEEEtran}
     \bibliography{ref_NNtracking}

     
\end{document}